%% file: main.tex
\documentclass[sigconf]{acmart}
\AtBeginDocument{%
  }

\setcopyright{acmlicensed}
\copyrightyear{2025}
\acmYear{2025}
\acmDOI{XXXXXXX.XXXXXXX}
\acmConference[KDD '26]{}{Aug 09--13, 2026}{Jeju, Korea}

\usepackage{enumitem}
\usepackage[ruled]{algorithm2e} 
\SetKwComment{Comment}{/* }{ */}

\usepackage{tabularx}
\usepackage{multirow}

\usepackage{amsthm}

\theoremstyle{definition}
\newtheorem{definition}{Definition}[section]

\theoremstyle{assumption}
\newtheorem{assumption}{Assumption}[section]

\newtheorem{theorem}{Theorem}[section]
\newtheorem{lemma}[theorem]{Lemma}

\theoremstyle{remark}

\begin{document}

\title{Cross-Treatment Effect Estimation for Multi-Category, Multi-Valued Causal Inference via Dynamic Neural Masking}

\author{Xiaopeng Ke}
\affiliation{%
  \institution{Didi Chuxing}
  \city{Beijing}
  \country{China}
}\email{kexiaopeng@didiglobal.com}

\author{Yihan Yu}
\affiliation{%
  \institution{Tsinghua University}
  \city{Beijing}
  \country{China}}
\email{yu-yh24@mails.tsinghua.edu.cn}

\author{Ruyue Zhang}
\affiliation{%
  \institution{Didi Chuxing}
  \city{Beijing}
  \country{China}
}\email{zhangruyue@didiglobal.com}

\author{Zhishuo Zhou}
\affiliation{%
  \institution{Didi Chuxing}
  \city{Beijing}
  \country{China}
}\email{zhouzhishuo@didiglobal.com}

\author{Fangzhou Shi}
\affiliation{%
  \institution{Didi Chuxing}
  \city{Beijing}
  \country{China}
}\email{arkshifangzhou@didiglobal.com}

\author{Chang Men}
\affiliation{%
  \institution{Didi Chuxing}
  \city{Beijing}
  \country{China}
}\email{menchang@didiglobal.com}

\author{Zhengdan Zhu}
\affiliation{%
  \institution{Didi Chuxing}
  \city{Beijing}
  \country{China}
}\email{zhuzhengdan@didiglobal.com}

\renewcommand{\shortauthors}{Xiaopeng Ke et al.}

\begin{abstract}


Counterfactual causal inference faces significant challenges when extended to multi-category, multi-valued treatments, where complex cross-effects between heterogeneous interventions are difficult to model. Existing methodologies remain constrained to binary or single-type treatments and suffer from restrictive assumptions, limited scalability, and inadequate evaluation frameworks for complex intervention scenarios.

We present XTNet, a novel network architecture for multi-category, multi-valued treatment effect estimation. Our approach introduces a cross-effect estimation module with dynamic masking mechanisms to capture treatment interactions without restrictive structural assumptions. The architecture employs a decomposition strategy separating basic effects from cross-treatment interactions, enabling efficient modeling of combinatorial treatment spaces. We also propose MCMV-AUCC, a suitable evaluation metric that accounts for treatment costs and interaction effects. Extensive experiments on synthetic and real-world datasets demonstrate that XTNet consistently outperforms state-of-the-art baselines in both ranking accuracy and effect estimation quality. The results of the real-world A/B test further confirm its effectiveness.
\vspace{-10pt}
\end{abstract}



\maketitle

\input{introduction}
\input{related_works}
\input{preliminary}

\input{method}

\input{metrics}

\input{evaluation}

\input{conclusion}


\bibliographystyle{ACM-Reference-Format}
\bibliography{sample-base}

\appendix
\input{appendix}

\end{document}

%% file: introduction.tex
\section{Introduction}

\begin{figure}[t]
    \centering
    \includegraphics[width=\linewidth]{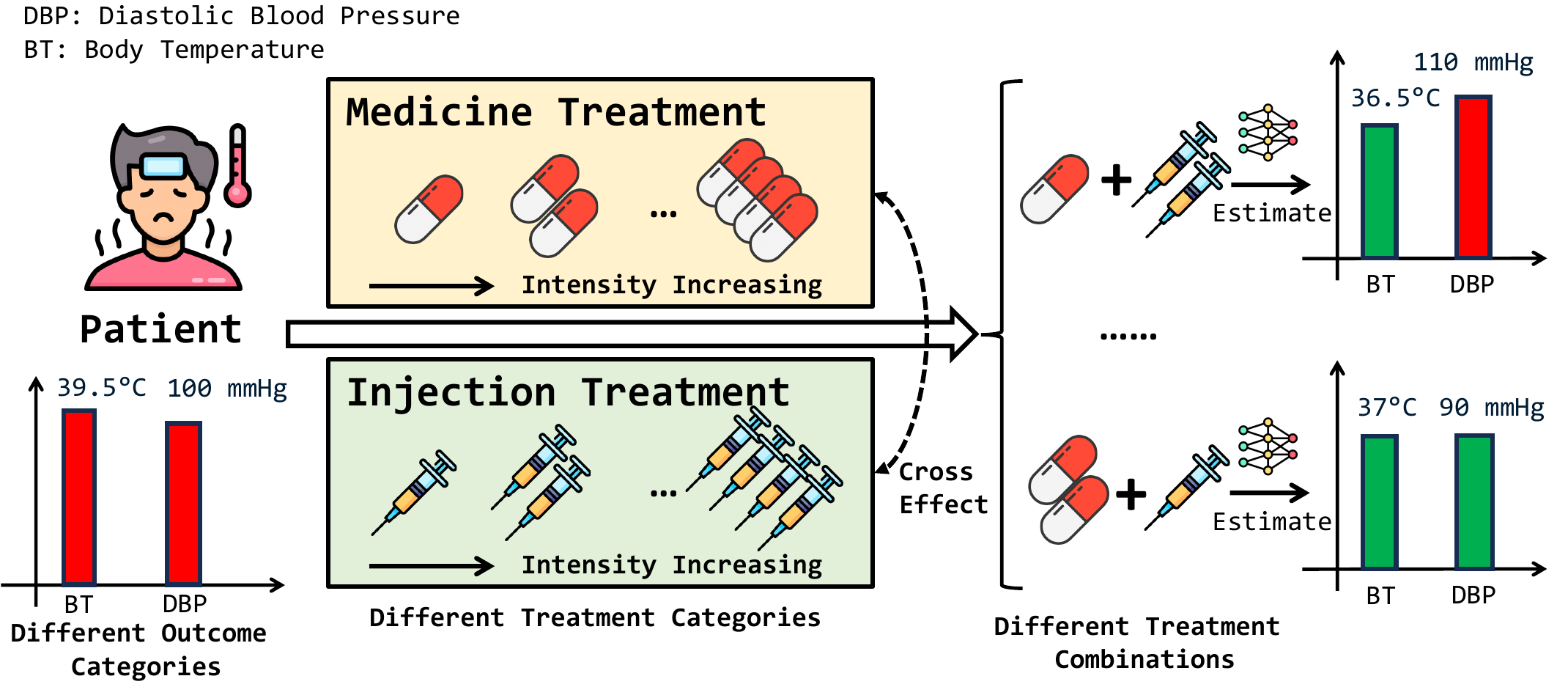}
    \caption{Multi-Category, Multi-Valued Treatment Scenario: An illustrative example with two distinct treatment types. The objective is to estimate patient outcomes (e.g., body temperature) under different treatment combinations.}
    \vspace{-20pt}
    \label{fig:mdmt-scene}
\end{figure}

Causal inference serves as a fundamental pillar for data-driven decision-making by enabling the identification and quantification of cause-and-effect relationships. It finds extensive applications across diverse domains, including healthcare \cite{prosperi2020causal}, e-commerce \cite{mondal2022aspire}, and ride-hailing platforms \cite{yu_multi-class_2024}. In these complex environments, accurately estimating intervention effects is paramount for optimizing resource allocation and strategic decision-making.

However, real-world interventions exhibit substantial complexity that stems from diverse treatment categories with multiple value options. Successfully deploying causal inference techniques in such applications necessitates the handling of sophisticated multi-category, multi-valued treatments, where treatments across different categories may exhibit intricate interactions that influence outcomes. Consequently, modeling individual treatment categories in isolation proves insufficient. To illustrate this complexity, consider the medical scenario depicted in Figure~\ref{fig:mdmt-scene}, where patients receive varying treatment intensities across two categories (injection and medication), affecting multiple outcomes (body temperature and blood pressure). The interdependence between these treatments can significantly influence final outcomes, making accurate treatment effect estimation critical for optimal patient care.

\begin{figure}
    \centering
    \includegraphics[width=\linewidth]{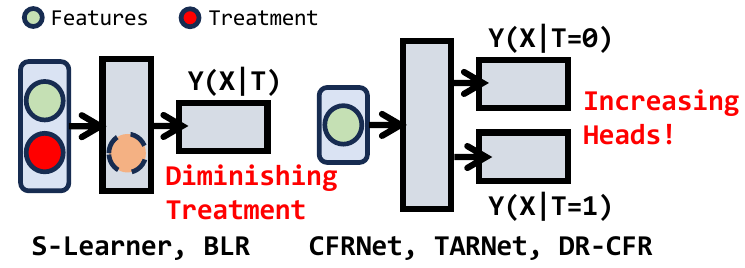}
    \vspace{-20pt}
    \caption{Current Model Architectures}
    \vspace{-20pt}
    \label{fig:recent-archs}
\end{figure}

The complexity inherent in real-world treatments presents formidable challenges for existing methodologies, which predominantly address simpler configurations such as binary \cite{rosenbaum1983central,athey2016recursive,wager2018estimation,johansson2016learning,shalit2017estimating} or single-category treatments \cite{li2022causal, schwab2018perfect}. We identify three primary challenges: (i) \textbf{Structural Scalability}, (ii) \textbf{Accurate Cross-Treatment Effect Estimation}, and (iii) \textbf{Proper Evaluation Metrics}. As illustrated in Figure~\ref{fig:recent-archs}, existing model architectures are ill-suited for complex treatment scenarios, leading to exponential parameter growth with treatment complexity and substantial computational overhead. While some approaches incorporate treatments as input features, this strategy often attenuates treatment effects and degrades estimation accuracy. Furthermore, all of these methods neglect cross-category interaction effects, compromising the precision of outcome predictions. Additionally, current evaluation metrics (e.g., Qini, AUUC) are inadequate for these complex scenarios because these metrics do not consider global ordering accuracy across different treatment combinations.

To address these limitations, we propose \textbf{XTNet} (Cross Treatment Network), a unified neural architecture specifically designed for multi-category, multi-valued treatment effect estimation. XTNet incorporates three key innovations: (1) \textbf{BasicNet} for establishing baseline treatment effects, (2) \textbf{EffectNet} for explicitly modeling cross-category interactions through dynamic masking mechanisms, and (3) \textbf{MaskNet} for generating treatment-specific parameter masks to enhance scalability. To enable proper evaluation in these complex settings, we introduce \textbf{MCMV-AUCC} (Multi-Category, Multi-Valued Area Under the Uplift Curve), a cost-aware metric specifically tailored for multi-category, multi-valued treatment scenarios. Through rigorous probabilistic analysis, we demonstrate that MCMV-AUCC achieves lower expected metric error by effectively incorporating treatment costs and marginal returns.

Our comprehensive experimental evaluation encompasses three synthetic datasets and one real-world dataset. Across all benchmarks, XTNet consistently achieves the lowest ranking error and highest MCMV-AUCC scores compared to state-of-the-art baselines. We further validate our design choices through extensive ablation studies examining loss terms and architectural components. Additionally, we present A/B test results from a production system that confirm XTNet's practical effectiveness.

Our primary contributions are as follows:
\begin{itemize}[leftmargin=*]
    \item We introduce and formalize the multi-category, multi-valued treatment effect estimation problem, representing the first systematic treatment of this challenging scenario in causal inference applications.
    \item We propose XTNet, a unified neural architecture that eliminates the need for separate models per treatment combination through innovative dynamic masking techniques. We also introduce MCMV-AUCC, a theoretically grounded evaluation metric proven to achieve lower expected error than existing approaches for multi-category, multi-valued scenarios.
    \item We demonstrate XTNet's superior performance through extensive empirical validation on both synthetic and real-world datasets, supported by comprehensive ablation studies and real-world A/B testing results.
\end{itemize}

%% file: related_works.tex
\section{Related Works}
\subsection{Deep Causal Inference}

Deep causal inference represents methods that leverage deep neural networks to estimate causal effects. Due to confounding bias, these methods typically aim to learn balanced representations. Counterfactual Regression (CFR) \cite{shalit_estimating_2017} proposed the Wasserstein and Maximum Mean Discrepancy (MMD) distance loss to balance hidden features and reduce selection bias in observational data. SITE \cite{Yao2018site} utilized local similarity to balance distributions through hard sample selection in mini-batches. Similarly, the disentangled representation technique is designed to separately model factors influencing treatment assignment and outcomes \cite{hassanpour2019drcfr}. GANITE \cite{yoon2018ganite} used Generative Adversarial Networks (GANs) to impute individual counterfactual outcomes while mitigating data bias. Causal Transformer \cite{melnychuk2022causal} also applied the adversarial learning technique and achieves counterfactual estimation from longitudinal data.


\subsection{Multi-Valued Treatment Causal Inference}


Perfect Match \cite{schwab2018perfect} augmented minibatches with propensity-matched neighbors, offering easy implementation in multi-treatment settings. Subsequently, NCoRE \cite{parbhoo2021ncore} estimated counterfactual effects of combined treatments through branched conditional representations with learned interaction modulators. MEMENTO \cite{mondal2022memento} used confounder matching representations and built a framework to handle uplift modeling in multi-treatment scenarios. Later, a Multi-gate Mixture-of-Experts based network \cite{sun2024m3tn} was proposed to address limitations of existing methods through efficient feature representation and explicit uplift reparameterization modules.



\subsection{Multi-Category Treatment Causal Inference}


An early approach \cite{zou2020counterfactual} leveraged low-dimensional latent treatment representations to decorrelate treatments from confounders, but their variational re-weighting (VSR) method cannot handle multi-valued treatment intensities within categories. Subsequent work \cite{ma2021multi} improved interpretability through disentangled representations while maintaining the binary treatment constraint. Their framework, though valuable for understanding causal structures, inherits the same limitation regarding treatment value granularity. SCP \cite{qian2021estimating} addressed data scarcity in multi-cause settings, but required untenable treatment ordering assumptions for concurrent interventions like combination therapies. The most recent MTMT \cite{wei2024multi} framework attempts to circumvent these issues by decomposing multi-valued treatment effects into binary treatment indicators and continuous intensities. This approach implicitly assumes effect scales are comparable across qualitatively different treatments.



%% file: preliminary.tex
\section{Preliminary}

\subsection{Notation and Definitions}

We establish our notation for the multi-category, multi-valued treatment framework. Let $\mathcal{D} = \{(x_i, \mathbf{t}_i, \mathbf{y}_i)\}_{i=1}^n$ denote our dataset of $n$ units, where:

\begin{itemize}[leftmargin=*]
    \item $x_i \in \mathcal{X} \subseteq \mathbb{R}^d$ represents the $d$-dimensional covariate vector for unit $i$
    \item $\mathbf{t}_i = (t_i^{(1)}, t_i^{(2)}, \ldots, t_i^{(m)}) \in \mathcal{T}$ represents the multi-category treatment vector
    \item $\mathbf{y}_i = (y_i^{(1)}, y_i^{(2)}, \ldots, y_i^{(s)}) \in \mathcal{Y} \subseteq \mathbb{R}^s$ represents the $s$-dimensional outcome vector
\end{itemize}

The treatment space is defined as $\mathcal{T} = \mathcal{T}^{(1)} \times \mathcal{T}^{(2)} \times \cdots \times \mathcal{T}^{(m)}$, where each category $k \in \{1, 2, \ldots, m\}$ has treatment space $\mathcal{T}^{(k)} = \{0, 1, 2, \ldots, a_k\}$. Here, $t^{(k)} = 0$ represents the no-treatment baseline for category $k$, while $t^{(k)} > 0$ represents different intensity levels of intervention.

For any treatment combination $\mathbf{t} \in \mathcal{T}$ and covariate $x \in \mathcal{X}$, we denote the potential outcome as $\mathbf{Y}(x, \mathbf{t}) = (Y^{(1)}(x, \mathbf{t}), \ldots, Y^{(s)}(x, \mathbf{t}))$, representing the outcome that would be observed if a unit with covariate $x$ received treatment combination $\mathbf{t}$.

\subsection{Problem Formulation}

Our primary objective is to estimate the \emph{conditional average treatment effect} (CATE) in the multi-category, multi-valued setting. Unlike binary treatments where CATE compares treated and control outcomes, multi-category treatments require comparing arbitrary treatment combinations.

\begin{definition}[Multi-Category CATE]
\label{def:mc_cate}
For covariates $x \in \mathcal{X}$ and treatment combinations $\mathbf{t}, \mathbf{t}' \in \mathcal{T}$, the multi-category conditional average treatment effect is:
\begin{equation}
\text{CATE}(x; \mathbf{t}, \mathbf{t}') = \mathbb{E}[\mathbf{Y}(x, \mathbf{t}) - \mathbf{Y}(x, \mathbf{t}') | X = x]
\end{equation}
\end{definition}

Of particular interest is the treatment effect relative to the no-treatment baseline $\mathbf{t}_0 = (0, 0, \ldots, 0)$:
\begin{equation}
\text{CATE}(x; \mathbf{t}) = \mathbb{E}[\mathbf{Y}(x, \mathbf{t}) - \mathbf{Y}(x, \mathbf{t}_0) | X = x]
\end{equation}


\subsection{Causal Assumptions}

To enable causal identification in observational data, we adopt the following standard assumptions from causal inference theory, extended to the multi-category setting:

\begin{assumption}[Unconfoundedness]
\label{ass:unconfoundedness}
Given observed covariates, treatment assignment is independent of potential outcomes:
\begin{equation}
\mathbf{Y}(x, \mathbf{t}) \perp\!\!\!\perp \mathbf{T} | X = x, \quad \forall \mathbf{t} \in \mathcal{T}, x \in \mathcal{X}
\end{equation}
This assumes no unmeasured confounding factors that simultaneously influence both treatment assignment and outcomes.
\end{assumption}

\begin{assumption}[Stable Unit Treatment Value Assumption (SUTVA)]
\label{ass:sutva}
The SUTVA comprises two components:
\begin{enumerate}[label=(\alph*)]
    \item \emph{No interference}: The potential outcome for unit $i$ depends only on unit $i$'s treatment, not on other units' treatments.
    \item \emph{No hidden variations}: For each treatment combination $\mathbf{t} \in \mathcal{T}$, there exists a single, well-defined potential outcome $\mathbf{Y}(x, \mathbf{t})$.
\end{enumerate}
\end{assumption}

\begin{assumption}[Overlap/Positivity]
\label{ass:overlap}
For all covariate values $x \in \mathcal{X}$ and treatment combinations $\mathbf{t} \in \mathcal{T}$:
\begin{equation}
P(\mathbf{T} = \mathbf{t} | X = x) > 0
\end{equation}
This ensures that every treatment combination has positive probability of occurrence across the covariate distribution.
\end{assumption}

%% file: method.tex
\begin{figure*}[ht]
    \centering
    \includegraphics[width=0.9\linewidth]{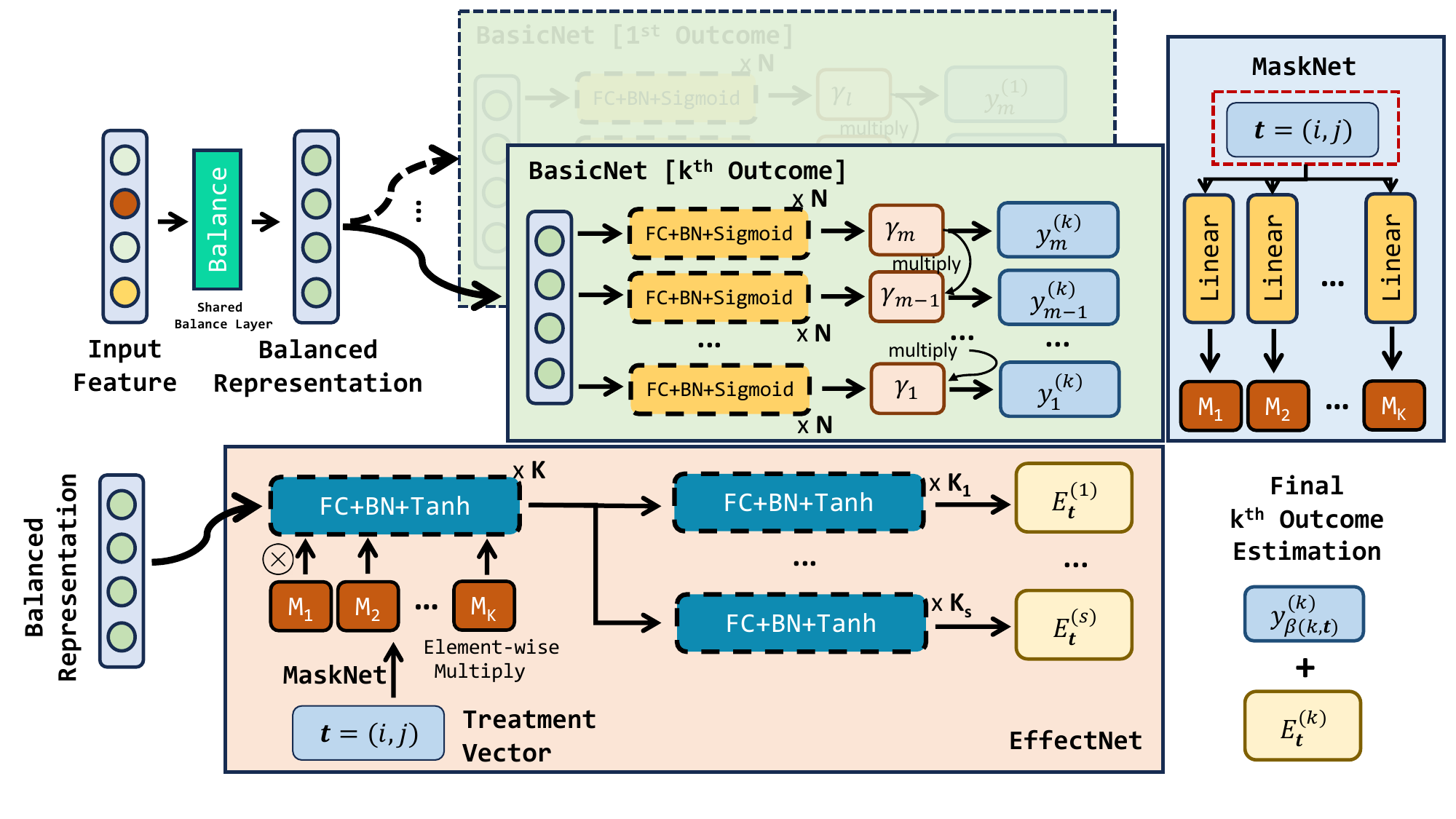}
    \caption{Architecture of XTNet for multi-treatment causal effect estimation. Our XTNet consists of three main components: BasicNet, EffectNet, and MaskNet. The BasicNet produces baseline outcome predictions using input features without cross-treatment effects. The EffectNet estimates cross-treatment effects from other treatments, which are then added to the baseline outcomes to obtain the final effect estimation. The MaskNet generates parameter masks for the EffectNet to construct treatment-specific masked networks. This masking mechanism provides flexibility for handling varying numbers of treatment categories.}
    \label{fig:xtnet-arch}
\end{figure*}

\section{Methodology}

In this section, we present the details of our proposed XTNet, including the network architecture and the design of training losses.

\subsection{Model Architecture Overview}

As shown in Figure~\ref{fig:xtnet-arch}, our proposed network architecture consists of three components: BasicNet, EffectNet, and MaskNet.

The final effect estimation for an input feature vector $x$ with a given treatment tuple $\textbf{t} = (t^{(1)},t^{(2)},...,t^{(m)})$ can be written as:
\begin{align}
    \hat{y}_{\textbf{t}}^{(i)} &= \text{BasicNet}^{(i)}(x,\textbf{t}) + \text{EffectNet}^{(i)}(x;M_{\textbf{t}}) \nonumber \\
    M_{\textbf{t}} &= \text{MaskNet}(t^{(1)},t^{(2)}, ..., t^{(m)}) \nonumber
\end{align}

where $\hat{y}_{\textbf{t}}^{(i)}$ denotes the $i$-th outcome, and $\text{BasicNet}^{(i)}$ denotes the $i$-th BasicNet (i.e., we train a separate BasicNet for each outcome category).

\subsection{BasicNet}

The BasicNet is designed to estimate dominant treatment effects without interference from the other treatments. This network can be trained using observational data filtered by treatment constraints (i.e., we only use samples that are primarily affected by a single treatment category). The function $\beta(k,\textbf{t})$ returns the index of the dominant treatment. It can be written as:
\begin{equation}
    \beta(k,\textbf{t}) = t_{\arg\max_{i} \mathbb{E}_{(x,\textbf{t},y)\sim\mathcal{D}}[y(x,\textbf{t}) - y(x,\textbf{t}^{-i})]}
\end{equation}
where $\textbf{t}^{-i}$ denotes $\textbf{t}$ after setting $t^{(i)}=0$.

Assume we have an $N$-layer Multi-Layer Perceptron as:
\begin{align}
    h_0(x) = \sigma(W_0x+b_0);\quad h_i(x) = \sigma(W_ih_{i-1}(x) + b_i) 
\end{align}

where $\sigma(\cdot)$ is the activation function. We then use a multi-head design to produce $m$ outputs, where each head is also a multi-layer perceptron. Let us denote the $i$-th head as $g_i(\cdot)$. For the $i$-th treatment:
\begin{align}
    \gamma_i(x) &= g_i(h_N(x)) \\
    \text{BasicNet}(x) &=
    \begin{cases}
      \gamma_i  & i = m\\
      \gamma_i \cdot \gamma_{i + 1}  & \text{otherwise}
    \end{cases} \\
    \text{BasicNet}(x,\textbf{t}) &= \text{BasicNet}(x)[\beta(\textbf{t})]
\end{align}

where the function $\beta(\textbf{t})$ represents the selection of the most effective treatment for the corresponding outcome category. Additionally, this chain design preserves monotonicity when required.

\subsection{EffectNet}

We design the EffectNet to estimate cross-treatment effects and it outputs $s$ cross-effects.

The EffectNet consists of two parts: (i) a backbone and (ii) effect heads. The backbone module adapts to different treatment combinations through weight masking. Here, we use MaskNet to generate multiple masks. Since cross-effects can be positive or negative, we utilize the tanh activation function at the output of each effect head. Assume we have a multi-layer perceptron $g^{\text{effect}}(\cdot)$ as the backbone module. We can formalize this function as:
\begin{align}
    h_0^{\text{effect}}(x) &= W_{0}^{\text{effect}}x + b_0^{\text{effect}}\\
    h_{i+1}^{\text{effect}}(x) &= W_{i+1}^{\text{effect}}h_{i}^{\text{effect}} + b_{i+1}^{\text{effect}}
\end{align}
where $h_i^{\text{effect}}$ is the $i$-th hidden feature vector, and $W_i^{\text{effect}}$ and $b_i^{\text{effect}}$ are the corresponding parameters.

Let $M_i^{W}$ and $M_i^{b}$ denote the weight and bias masks produced by MaskNet, respectively. We then perform element-wise multiplication on the backbone of EffectNet. Thus, we can formalize the output of the masked perceptron as follows:
\begin{align}
    \tilde{h}_0^{\text{effect}}(x) &= \tilde{W}_{0}^{\text{effect}}x + \tilde{b}_0^{\text{effect}}\\
    \tilde{W}_i^{\text{effect}} &= M_i^{W} \odot W_i^{\text{effect}} \\
    \tilde{b}_i^{\text{effect}} &= M_i^{b} \odot b_i^{\text{effect}}  \\
    \tilde{h}_{i+1}^{\text{effect}}(x) &= \tilde{W}_{i+1}^{\text{effect}}\tilde{h}_{i}^{\text{effect}} + \tilde{b}_{i+1}^{\text{effect}}
\end{align}
where $\odot$ denotes element-wise multiplication.

Thus, we can formalize the output of EffectNet with a given feature input $x$ as: 
\begin{align}
    \text{EffectNet}(x) &= g^{\text{effect}}(\tilde{h}_{K}^{\text{effect}}(x)) \\
    g_0^{\text{effect}}(h) &= \sigma( W^g_0h + b^g_0)\\
    g^{\text{effect}}(h) &= \sigma( W_K^gg_{K-1}^{\text{effect}}(h) + b_{K-1}^{g}) 
\end{align}
where $K$ denotes the number of layers in the head module, $\sigma(\cdot)$ denotes the activation function, $g^{\text{effect}}(\cdot)$ is the head module, and $W^g_i$ and $b^g_i$ denote the weight and bias of the $i$-th layer, respectively.

\subsection{MaskNet}

Our MaskNet is designed to modulate the behavior of EffectNet by conditioning on different treatment combinations. It comprises multiple independent linear layers, each generating a mask tailored to a corresponding layer in the EffectNet.

To formalize the computation of MaskNet, suppose there are $m$ treatment categories and the EffectNet consists of a $K$-layer backbone MLP. Given an input treatment combination vector $\mathbf{t} = (t^{(1)}, t^{(2)}, \dots, t^{(m)})$, MaskNet generates a set of $K$ masks via separate linear transformations:
\begin{align}
     M_i &= W^{\text{mask}}_i \mathbf{t}^\top + b^{\text{mask}}_i, \quad i=1,2,...,K\\
    \text{MaskNet}(\mathbf{t}) &= (M_1, M_2, \dots, M_K)
\end{align}

Each mask $M_i$ corresponds to the $i$-th layer of the EffectNet and contains parameter-specific masks for that layer. If the layer is a linear layer, the mask is expressed as a tuple $M_i = (M^W_i, M^b_i)$, where $M^W_i$ and $M^b_i$ denote the weight and bias masks, respectively.

\subsection{Loss Design}

Our training loss consists of two components: (i) the factual loss, which measures the error in outcome prediction for observed treatments, and (ii) the imbalance loss, which addresses selection bias in observational data. The imbalance loss is commonly used in state-of-the-art methods to mitigate this bias. 
The overall loss can be formulated as:
\begin{equation}
    \mathcal{L} = \lambda_1 \cdot \mathcal{L}_{\text{factual}} + \lambda_2 \cdot \mathcal{L}_{\text{imb}}
\end{equation}

where $\lambda_1$ and $\lambda_2$ are coefficients controlling the weights.

For the factual loss, we apply binary cross-entropy loss to measure the effect estimation error. It can be formalized as follows:
\begin{equation}
    \mathcal{L}_{\text{factual}} = -\frac{1}{ns}\sum_{i=1}^n\sum_{j=1}^{s} \left[y^{(j)}_{i,\textbf{t}_i}\log(\hat{y}^{(j)}_{i,\textbf{t}_i}) + (1-y^{(j)}_{i,\textbf{t}_i})\log(1-\hat{y}^{(j)}_{i,\textbf{t}_i})\right]
\end{equation}

For the imbalance loss, we use the Sinkhorn distance to align the feature distributions across different treatment groups. This loss can be expressed as:
\begin{equation}
    \mathcal{L}_{\text{imb}} = \sum_{\textbf{t}_1\neq \textbf{t}_2} \text{disc}(\{h_{i,\textbf{t}}\}_{\textbf{t}=\textbf{t}_1}, \{h_{i,\textbf{t}}\}_{\textbf{t}=\textbf{t}_2})
\end{equation}

where $\{h_{i,\textbf{t}}\}$ denotes the hidden features for a given treatment $\textbf{t}$ across all samples, and $\text{disc}(\cdot, \cdot)$ is the discrepancy computation function. The total loss can then be rewritten as:
\begin{align}
\mathcal{L} &= -\lambda_1 \cdot \frac{1}{ns}\sum_{i=1}^n\sum_{j=1}^{s} \left[y^{(j)}_{i,\textbf{t}_i}\log(\hat{y}^{(j)}_{i,\textbf{t}_i}) + (1-y^{(j)}_{i,\textbf{t}_i})\log(1-\hat{y}^{(j)}_{i,\textbf{t}_i})\right] \nonumber \\
    &\quad + \lambda_2 \cdot \sum_{\textbf{t}_1\neq \textbf{t}_2} \text{disc}(\{h_{i,\textbf{t}}\}_{\textbf{t}=\textbf{t}_1}, \{h_{i,\textbf{t}}\}_{\textbf{t}=\textbf{t}_2})
\end{align}

\begin{algorithm}[t]
    \caption{XTNet Training Algorithm}\label{alg:training}
    \KwData{Training Data $\mathcal{D}_{\text{train}} = \{(x_i,y_i,t_i)\}$, Learning Rate $\eta$, XTNet Parameters $\theta = (\theta_{\text{BasicNet}},\theta_{\text{EffectNet}}, \theta_{\text{MaskNet}})$, Loss Coefficients $\lambda_1,\lambda_2$, Batch Size $B$, Max Epochs $E_{\max}$}
    \KwResult{Trained XTNet Parameters $\theta^*$}
    $e := 0$\;
    \While{$e < E_{\max}$}{
        $(X_{\text{batch}},Y_{\text{batch}}, T_{\text{batch}}) \leftarrow \text{fetch\_batch}(\mathcal{D}_{\text{train}}, B)$\;
        \tcp{Filter samples with isolated treatments}
        $(X'_{\text{batch}},Y'_{\text{batch}}, T'_{\text{batch}}) \leftarrow \text{filter}(X_{\text{batch}},Y_{\text{batch}}, T_{\text{batch}})$\;
        $\hat{Y}'_{\text{batch}} \leftarrow \text{BasicNet}(X'_{\text{batch}},T'_{\text{batch}};\theta_{\text{BasicNet}})$\;
        $\mathcal{L}_{\text{BasicNet}} \leftarrow \mathcal{L}_{\text{factual}}(\hat{Y}'_{\text{batch}}, Y'_{\text{batch}})$\;
        \tcp{Update BasicNet}
        $\theta_{\text{BasicNet}} \leftarrow \theta_{\text{BasicNet}} - \eta \cdot \nabla_{\theta_{\text{BasicNet}}} \mathcal{L}_{\text{BasicNet}}$\;
        $M_{\text{batch}} \leftarrow \text{MaskNet}(T_{\text{batch}};\theta_{\text{MaskNet}})$\;
        $\hat{Y}_{\text{batch}} \leftarrow \text{BasicNet}(X_{\text{batch}},T_{\text{batch}};\theta_{\text{BasicNet}}) + \text{EffectNet}(X_{\text{batch}};\theta_{\text{EffectNet}},M_{\text{batch}})$\;
        $\mathcal{L}_{\text{train}} \leftarrow \lambda_1 \cdot \mathcal{L}_{\text{factual}}(\hat{Y}_{\text{batch}}, Y_{\text{batch}}) + \lambda_2 \cdot \mathcal{L}_{\text{imb}}(H_{\text{batch}}, T_{\text{batch}})$\;
        \tcp{Update EffectNet and MaskNet}
        $\theta_{\text{EffectNet}} \leftarrow \theta_{\text{EffectNet}} - \eta \cdot \nabla_{\theta_{\text{EffectNet}}} \mathcal{L}_{\text{train}}$\;
        $\theta_{\text{MaskNet}} \leftarrow \theta_{\text{MaskNet}} - \eta \cdot \nabla_{\theta_{\text{MaskNet}}} \mathcal{L}_{\text{train}}$\;
        $e \leftarrow e + 1$\;
    }
\end{algorithm}

\subsection{Training Algorithm}

As shown in Algorithm~\ref{alg:training}, we present the training pipeline of XTNet. The algorithm iteratively optimizes the model parameters using mini-batch stochastic gradient descent. During each epoch, a batch of training data is sampled and preprocessed to select samples with isolated treatments for training the BasicNet. The BasicNet component first estimates outcomes for the filtered batch and is updated by minimizing the factual loss. Next, MaskNet generates treatment-specific masks, which, along with EffectNet, refine the outcome predictions. The total training loss is computed as a weighted sum of the factual loss and the imbalance loss, controlled by coefficients $\lambda_1$ and $\lambda_2$. The parameters of EffectNet and MaskNet are then updated to minimize this total loss. This process is repeated for a predefined number of epochs.

%% file: metrics.tex
\section{Metric Design}

In this paper, we propose a new metric named MCMV-AUCC which is more suitable for multi-category, multi-valued treatment scenarios. We also present theoretical proofs demonstrating the advantages of this proposed metric using rigorous probabilistic analysis.

First, we establish our probabilistic framework. Let $(\Omega, \mathcal{F}, \mathbb{P})$ be a probability space where samples $X \sim P_X$ are drawn from a distribution over the sample space $\mathcal{X}$. Each sample can receive multi-valued treatments $T \in \mathcal{T}$, producing stochastic costs $C(X,T)$ and outcomes $Y(X,T)$ with conditional distributions given $(X,T)$.

\begin{assumption}[Stochastic Monotonicity]
\label{ass:stoch_monotonicity}
For any treatments $t_i$ and $t_j$ with indices $i \leq j$, we have: 
\begin{align*}
\mathbb{E}[C(X,t_i)] &\leq \mathbb{E}[C(X,t_j)] \\
\mathbb{E}[Y(X,t_i)] &\leq \mathbb{E}[Y(X,t_j)]
\end{align*}
for all $X \sim P_X$, where expectations are taken over the conditional distributions of costs and outcomes.
\end{assumption}

\begin{assumption}[Stochastic Diminishing Returns]
\label{ass:stoch_convexity}
For any treatment index $i \geq 1$ and sample $X$, the marginal expected return on investment is decreasing:
\begin{align*}
\frac{\mathbb{E}[Y(X,t_i)]-\mathbb{E}[Y(X,t_{i-1})]}{\mathbb{E}[C(X,t_i)]-\mathbb{E}[C(X,t_{i-1})]}  > \frac{\mathbb{E}[Y(X,t_{i+1})]-\mathbb{E}[Y(X,t_{i})]}{\mathbb{E}[C(X,t_{i+1})]-\mathbb{E}[C(X,t_{i})]}
\end{align*}
where treatments are ordered by expected cost: 
$$\mathbb{E}[C(X,t_0)] = 0 < \mathbb{E}[C(X,t_1)] \leq \mathbb{E}[C(X,t_2)] \leq \cdots$$
and $t_0$ represents the no-treatment baseline with $\mathbb{E}[Y(X,t_0)] = 0$.
\end{assumption}

\begin{assumption}[Bounded Moments and Regularity]
\label{ass:bounded_moments}
For all treatments $t \in \mathcal{T}$ and samples $X \sim P_X$:
\begin{align*}
\sup_{t \in \mathcal{T}} \mathbb{E}[C(X,t)^2] &< \infty, \quad \inf_{t \in \mathcal{T} \setminus \{t_0\}} \mathbb{E}[C(X,t)] > 0 \\
\sup_{t \in \mathcal{T}} \mathbb{E}[Y(X,t)^2] &< \infty, \quad \inf_{t \in \mathcal{T} \setminus \{t_0\}} \mathbb{E}[Y(X,t)] > 0
\end{align*}
This ensures well-defined variances, concentration inequalities, and avoids degeneracies.
\end{assumption}

\begin{definition}[Stochastic Ideal Allocation]
\label{def:stoch_ideal}
Given budget $B$ and sample distribution $P_X$, the stochastic ideal allocation $\Pi^*(P_X, B)$ is a probability measure over allocation policies that maximizes expected outcome:
\begin{align*}
\Pi^*(P_X, B) = \arg\max_{\pi \in \mathcal{P}_B} &\mathbb{E}_{X \sim P_X} \mathbb{E}_{\pi} \left[ \sum_{(x,t) \in \text{supp}(\pi)} Y(x,t) \right]
\end{align*}
subject to the budget constraint:
\[
\mathbb{E}_{X \sim P_X} \mathbb{E}_{\pi} \left[ \sum_{(x,t) \in \text{supp}(\pi)} C(x,t) \right] \leq B
\]
where $\mathcal{P}_B$ is the set of feasible allocation policies under budget constraint $B$.
\end{definition}

\begin{definition}[Expected Metric Error]
\label{def:expected_metric_error}
For a metric $M$ with allocation policy $\Pi_M$, the expected metric error is:
\begin{align*}
\mathcal{E}(M) = \mathbb{E}_{P_X, F} \bigg[ &\int_0^{B_{\max}} \Big| \mathbb{E}_{\Pi^*}[\text{Outcome}(B)] - \mathbb{E}_{\Pi_M}[\text{Outcome}(B)] \Big| dB \bigg]
\end{align*}
where $F$ represents the distribution of predictive models and expectations are taken over both sample and model uncertainty. 
\end{definition}

\begin{theorem}[Stochastic Dominance Principle]
\label{thm:stoch_dominance}
If allocation policy $\Pi_1$ stochastically dominates $\Pi_2$ in the sense that 
\[
\mathbb{E}_{\Pi_1}[\text{Outcome}(B)] \geq \mathbb{E}_{\Pi_2}[\text{Outcome}(B)]
\]
for all budgets $B$ and all sample realizations, then $\mathcal{E}(M_1) \leq \mathcal{E}(M_2)$.
\end{theorem}

\begin{proof}
Let $O^*(B)$, $O_1(B)$, and $O_2(B)$ denote the expected outcomes under policies $\Pi^*$, $\Pi_1$, and $\Pi_2$ respectively. By hypothesis, $O_1(B) \geq O_2(B)$ for all $B$, and by optimality, $O^*(B) \geq O_1(B) \geq O_2(B)$. Therefore,
\begin{align*}
\mathcal{E}(M_1) - \mathcal{E}(M_2) = \mathbb{E}_{P_X, F} \left[ \int_0^{B_{\max}} O_2(B) - O_1(B) \, dB \right] \leq 0
\end{align*}
\end{proof}

\begin{definition}[AUCC Allocation Policy]
\label{def:aucc_policy}
The AUCC policy $\Pi_{\text{AUCC}}$ ranks treatments by expected RoI:
\[
\rho(x,t) = \frac{\mathbb{E}[Y(x,t)]}{\mathbb{E}[C(x,t)]}
\]
and allocates budget to treatments in decreasing order of $\rho(x,t)$.
\end{definition}

\begin{lemma}[RoI Ordering]
\label{lem:roi_ordering}
Under Assumption \ref{ass:stoch_convexity}, for any sample $x$ and treatments $t_i, t_j$ with $i < j$:
\[
\rho(x,t_i) > \rho(x,t_j)
\]
\end{lemma}

\begin{proof}
Let $\Delta Y_k = \mathbb{E}[Y(x,t_k)] - \mathbb{E}[Y(x,t_{k-1})]$ and $\Delta C_k = \mathbb{E}[C(x,t_k)] - \mathbb{E}[C(x,t_{k-1})]$. By Assumption \ref{ass:stoch_convexity}, $\frac{\Delta Y_k}{\Delta C_k} > \frac{\Delta Y_{k+1}}{\Delta C_{k+1}}$ for all $k \geq 1$. Since $\rho(x,t_j) = \frac{\sum_{k=1}^j \Delta Y_k}{\sum_{k=1}^j \Delta C_k}$ is a weighted average of decreasing marginal returns, we have $\frac{\Delta Y_j}{\Delta C_j} < \rho(x,t_j) < \frac{\Delta Y_1}{\Delta C_1} = \rho(x,t_1)$. Therefore, $\rho(x,t_i) > \rho(x,t_j)$ for $i < j$.
\end{proof}

\begin{theorem}[Qini vs AUCC]
\label{thm:qini_aucc_prob}
Under Assumptions \ref{ass:stoch_monotonicity}, \ref{ass:stoch_convexity}, and \ref{ass:bounded_moments}, the AUCC policy achieves lower expected metric error than the Qini policy: 
\[
\mathcal{E}(M_{\text{AUCC}}) \leq \mathcal{E}(M_{\text{Qini}})
\]
\end{theorem}

\begin{proof}
See Appendix~\ref{app:qini-aucc}.
\end{proof}

\begin{definition}[MV-AUCC Allocation Policy]
\label{def:mv_aucc_policy}
The MV-AUCC policy $\Pi_{\text{MV-AUCC}}$ ranks treatments by expected marginal RoI:
\begin{align*}
\mu(x,t_i) = \frac{\mathbb{E}[Y(x,t_i)] - \mathbb{E}[Y(x,t_{i-1})]}{\mathbb{E}[C(x,t_i)] - \mathbb{E}[C(x,t_{i-1})]}, \quad i \geq 1
\end{align*}
\end{definition}

\begin{theorem}[AUCC vs MV-AUCC]
\label{thm:aucc_mv_aucc_exchange}
Under Assumptions \ref{ass:stoch_monotonicity}, \ref{ass:stoch_convexity}, and \ref{ass:bounded_moments}, MV-AUCC achieves lower or equal expected metric error: 
\[
\mathcal{E}(M_{\text{MV-AUCC}}) \leq \mathcal{E}(M_{\text{AUCC}})
\]
\end{theorem}

\begin{proof}
Consider an AUCC allocation $S_A = \{(x_j,t_{i_j})\}$ ordered by decreasing total RoI. For any treatment $(x_j, t_{i_j})$ with $i_j > 1$, the marginal components $\{(x_j, t_2), \ldots, (x_j, t_{i_j})\}$ have marginal RoIs $\mu(x_j, t_2) > \cdots > \mu(x_j, t_{i_j})$ by Assumption \ref{ass:stoch_convexity}.

The average marginal RoI of the components beyond the first is:
\[
\bar{\mu}_j = \frac{\mathbb{E}[Y(x_j, t_{i_j})] - \mathbb{E}[Y(x_j, t_1)]}{\mathbb{E}[C(x_j, t_{i_j})] - \mathbb{E}[C(x_j, t_1)]} = \frac{\sum_{k=2}^{i_j} \Delta Y_k}{\sum_{k=2}^{i_j} \Delta C_k} < \mu(x_j, t_2)
\]

Since MV-AUCC optimally allocates budget to the globally highest marginal RoI treatments, there exists a feasible exchange that replaces the marginal components with RoI $\bar{\mu}_j$ with treatments having marginal RoI at least $\mu(x_j, t_2) > \bar{\mu}_j$. This exchange strictly improves the outcome while satisfying the budget constraint. Iterating this process over all treatments in $S_A$ yields $\mathbb{E}_{\Pi_{\text{MV-AUCC}}}[O(B)] \geq \mathbb{E}_{\Pi_{\text{AUCC}}}[O(B)]$ for all $B$. By Theorem \ref{thm:stoch_dominance}, $\mathcal{E}(M_{\text{MV-AUCC}}) \leq \mathcal{E}(M_{\text{AUCC}})$.
\end{proof}

\begin{definition}[MCMV-AUCC Allocation Policy]
\label{def:mcmv_aucc_policy}
For multi-category treatments $\mathbf{t} = (t^{(1)}, t^{(2)}, \ldots, t^{(m)})$, define the grouping function:
\[
Q(\mathbf{t}) = \sum_{j=1}^k t^{(j)}
\]

For each intensity level $g$, define the group-averaged outcome:
\[
\bar{Y}(x,q) = \mathbb{E}_{\mathbf{t}: Q(\mathbf{t})=q}[Y(x,\mathbf{t})]
\]

The MCMV-AUCC policy uses grouped marginal RoI:
\[
\gamma(x,q) = \frac{\bar{Y}(x,q) - \bar{Y}(x,q-1)}{\bar{C}(x,q) - \bar{C}(x,q-1)}
\]
where $\bar{C}(x,q) = \mathbb{E}_{\mathbf{t}: Q(\mathbf{t})=q}[C(x,\mathbf{t})]$.
\end{definition}

\begin{assumption}[Grouping Coherence]
\label{ass:grouping_coherence}
The grouping function preserves essential monotonicity:
\begin{align*}
q_1 < q_2 &\Rightarrow \bar{Y}(x,q_1) \leq \bar{Y}(x,q_2) \\
q_1 < q_2 &\Rightarrow \bar{C}(x,q_1) \leq \bar{C}(x,q_2)
\end{align*}
and the grouped marginal returns are decreasing: $\gamma(x,q) > \gamma(x,q+1)$ for all $g$.
\end{assumption}

\begin{theorem}[MV-AUCC vs MCMV-AUCC]
\label{thm:mv_aucc_mcmv_aucc_dimension}
Under Assumptions \ref{ass:stoch_monotonicity}, \ref{ass:stoch_convexity}, \ref{ass:bounded_moments}, and \ref{ass:grouping_coherence}, MCMV-AUCC achieves lower expected metric error: 
\[
\mathcal{E}(M_{\text{MCMV-AUCC}}) \leq \mathcal{E}(M_{\text{MV-AUCC}})
\]
with strict inequality when treatment interactions exist.
\end{theorem}

\begin{proof}

The superiority of MCMV-AUCC stems from MV-AUCC’s systematic ordering errors caused by erroneous marginal RoI estimations within incomparable treatment groups.

For outcome functions with cross-category interactions $\mathbb{E}[Y(x,\mathbf{t})] = \sum_{j=1}^k f_j(x, t^{(j)}) + \sum_{i<j} g_{ij}(x, t^{(i)}, t^{(j)})$, consider treatments $\mathbf{t}_A = (q, 0, \ldots, 0)$ (concentrated) and $\mathbf{t}_B$ (distributed across categories) with the same total intensity. While $\mathbf{t}_A$ may exhibit high marginal returns $\mu^*(x, \mathbf{t}_A)$, $\mathbf{t}_B$ benefits from interaction terms $\sum_{i<j} g_{ij}(x, t_B^{(i)}, t_B^{(j)})$ that MV-AUCC's marginal calculation ignores, leading to systematic undervaluation of interaction-rich treatments.

MCMV-AUCC eliminates ordering errors by creating a total ordering through the grouping function $Q(\mathbf{t}) = \sum_{j=1}^k t^{(j)}$. Within each intensity group, the unified metric $\gamma(x,q) = \frac{\bar{Y}(x,q) - \bar{Y}(x,q-1)}{\bar{C}(x,q) - \bar{C}(x,q-1)}$ uses group-averaged outcomes $\bar{Y}(x,q) = \mathbb{E}_{\mathbf{t}: Q(\mathbf{t})=q}[Y(x,\mathbf{t})]$ that automatically incorporate all interaction patterns, providing consistent global ranking information. Let $\mathcal{E}_q$ denote treatments where MV-AUCC's marginal ranking contradicts true outcome ordering within intensity group $q$. When interaction effects exist with magnitude bounded by $G$, the aggregate improvement satisfies:
\[
\mathbb{E}_{\Pi_{\text{MCMV}}}[O(B)] - \mathbb{E}_{\Pi_{\text{MV}}}[O(B)] \geq \sum_{q} \sum_{(\mathbf{t}_i, \mathbf{t}_j) \in \mathcal{E}_q} \phi_q \cdot \delta_{ij}
\]
where $\phi_q$ represents the frequency of ordering errors and $\delta_{ij}$ measures the outcome difference between correctly and incorrectly ordered treatments. By Theorem \ref{thm:stoch_dominance}, $\mathcal{E}(M_{\text{MCMV-AUCC}}) \leq \mathcal{E}(M_{\text{MV-AUCC}})$. The complete analysis with concrete bounds is provided in Appendix~\ref{app:mv-aucc-mcmv-aucc}.
\end{proof}

%% file: evaluation.tex
\section{Experimental Evaluation}

\begin{table}[h]
    \centering
    \begin{tabularx}{0.43\textwidth}{@{}clccc@{}}
    \toprule
    \textbf{Dataset}
    &\textbf{Method} 
    &\textbf{Ranking Error} 
    &\textbf{MCMV-AUCC}\\
    \midrule
    \multirow{7.5}{*}{Synthetic-1} & 
    BLR
    & ${0.2613 \pm 0.0226}$ & $0.3098$ &\\
    \cmidrule(lr){2-5}
    & 
    CFRNet
    & $0.2595 \pm 0.0309$ & $0.0521$  &\\
    \cmidrule(lr){2-5} &
    TARNet
    & $0.2514 \pm 0.0276$ & ${0.1348}$  &\\
    \cmidrule(lr){2-5} &
    DR-CFR
    & $0.2429 \pm 0.0280$ & ${0.2137}$  &\\
    \cmidrule(lr){2-5} &
    XTNet
    & $\textbf{0.2272} \pm 0.0433$ & $\mathbf{0.5375}$  &\\
    \midrule

    \multirow{7.5}{*}{Synthetic-2} & 
    BLR
    & ${0.2740 \pm 0.0300}$ & $0.0504$ &\\
    \cmidrule(lr){2-5}
    & 
    CFRNet
    & $0.2592 \pm 0.0310$ & $0.0346$  &\\
    \cmidrule(lr){2-5} &
    TARNet
    & $0.2632 \pm 0.0295$ & ${-0.0004}$  &\\
    \cmidrule(lr){2-5} &
    DR-CFR
    & $0.2588 \pm 0.0233$ & ${0.0102}$  &\\
    \cmidrule(lr){2-5} &
    XTNet
    & $\textbf{0.2465} \pm 0.0380$ & $\mathbf{0.0529}$  &\\
    \midrule
    \multirow{7.5}{*}{Synthetic-3} & 
    BLR
    & ${0.2868 \pm 0.0268}$ & $0.0718$ &\\
    \cmidrule(lr){2-5}
    & 
    CFRNet
    & $0.2767 \pm 0.0310$ & $0.0520$  &\\
    \cmidrule(lr){2-5} &
    TARNet
    & $0.2737 \pm 0.0332$ & ${0.0500}$  &\\
    \cmidrule(lr){2-5} &
    DR-CFR
    & $0.2848 \pm 0.0244$ & ${0.0207}$  &\\
    \cmidrule(lr){2-5} &
    XTNet
    & $\textbf{0.2726} \pm 0.0335$ & $\mathbf{0.0834}$  &\\
    \bottomrule
    \end{tabularx}
    \caption{Performance comparison on synthetic datasets. XTNet consistently achieves the lowest ranking error and highest MCMV-AUCC across all datasets.}
    \vspace{-20pt}
    \label{tab:perf-syn}
\end{table}

\begin{table}[h]
\centering
    \begin{tabularx}{0.43\textwidth}{XXc}
    \toprule
    \textbf{Method} 
    &\textbf{Ranking Error} 
    &\textbf{MCMV-AUCC}\\
    \midrule
    BLR
    & ${0.1436 \pm 0.1943}$ & $0.3840$ \\
    \cmidrule(lr){1-3}
    CFRNet
    & $0.2357 \pm 0.2384$ & $-0.4148$  \\
    \cmidrule(lr){1-3} 
    TARNet
    & $0.1221 \pm 0.2235$ & $-0.0072$ \\
    \cmidrule(lr){1-3} 
    DR-CFR
    & $0.1117 \pm 0.2089$ & $1.0885$ \\
    \cmidrule(lr){1-3} 
    XTNet
    & $\textbf{0.1100} \pm 0.2073$ & $\mathbf{1.1701}$ \\
    \bottomrule
    \end{tabularx}
    \caption{Performance comparison on real-world ride-hailing dataset.}
    \vspace{-30pt}
    \label{tab:perf-real}
\end{table}

Our empirical evaluation addresses the following research questions:
\begin{itemize}[leftmargin=*]
    \item \textbf{RQ1:} How does XTNet perform compared to state-of-the-art methods on multi-category, multi-valued treatment effect estimation?
    \item \textbf{RQ2:} What is the contribution of each component in the XTNet architecture?
    \item \textbf{RQ3:} How does the proposed MCMV-AUCC metric compare to existing evaluation approaches?
\end{itemize}

\subsection{Experimental Setup}

\textbf{Datasets.} We evaluate our approach on both synthetic and real-world datasets to ensure comprehensive assessment across diverse scenarios.

\textit{Synthetic Datasets:} We construct three synthetic datasets (Syn-1, Syn-2, Syn-3) with varying complexity in treatment interactions. Each dataset contains 8-dimensional feature vectors with multi-category treatments exhibiting different cross-effect patterns. To simulate realistic observational bias, we generate 50\% observational data (with treatment selection bias) and 50\% randomized controlled trial (RCT) data. Each dataset comprises 64,000 training samples and 8,000 test samples.

\textit{Real-world Dataset:} We collected data from a ride-hailing platform's coupon experiment involving 546,262 passengers over one week. The treatment space consists of two service categories (Standard and Premium rides) with five discount levels (0\%, 5\%, 10\%, 15\%, 20\%) each, creating a $5^2$ multi-category, multi-valued treatment structure. We only collected 240 marketing environment features that are not related to the customer.

\textbf{Baselines.} We compare against four representative causal inference methods adapted for multi-category treatments: BLR~\cite{johansson2016learning}, CFRNet~\cite{shalit2017estimating}, TARNet~\cite{shalit2017estimating}, and DRCFR~\cite{hassanpour_learning_2020}. Since these methods were originally designed for binary treatments, we extend their architectures to handle our multi-category setting through separate outcome heads for each treatment combination.

\textbf{Evaluation Metrics.} We employ two complementary metrics: (1) Ranking Error (Spearman's Footrule Distance) measuring the deviation from optimal treatment ranking, and (2) our proposed MCMV-AUCC capturing cost-adjusted treatment effectiveness in multi-dimensional treatment spaces.

\setlength{\tabcolsep}{1pt}
\begin{table*}[t]
    \begin{tabularx}{\textwidth}{@{}ccc>{\compress}c@{}}
    \toprule
    \textbf{Synthetic-1} 
    &\textbf{Synthetic-2} 
    &\textbf{Synthetic-3} \\
    \midrule
    
    \begin{minipage}{0.33\textwidth}
      \centering
      \includegraphics[width=\linewidth, height=40mm]{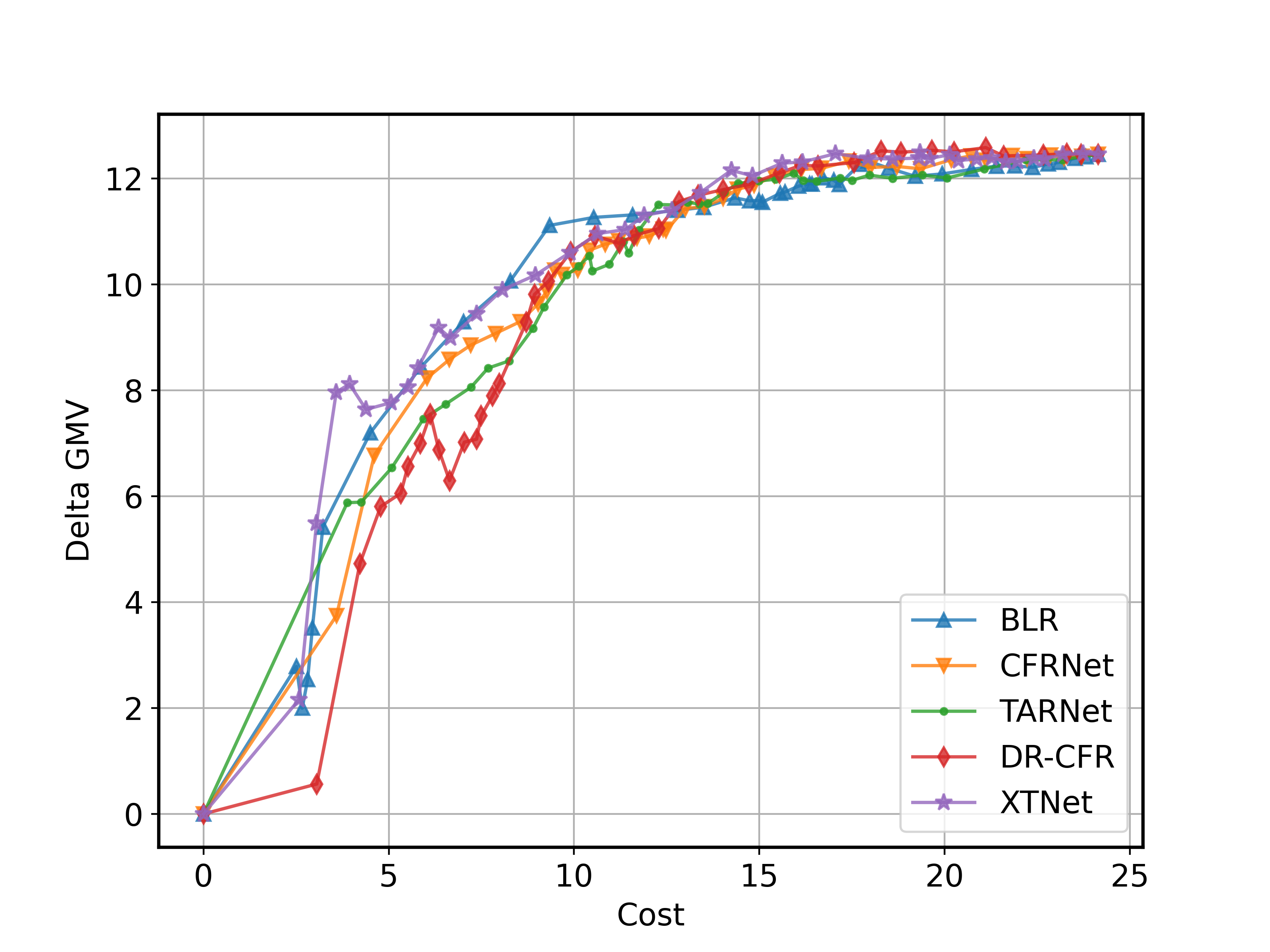}
    \end{minipage}
    &
    \begin{minipage}{0.33\textwidth}
      \centering
      \includegraphics[width=\linewidth, height=40mm]{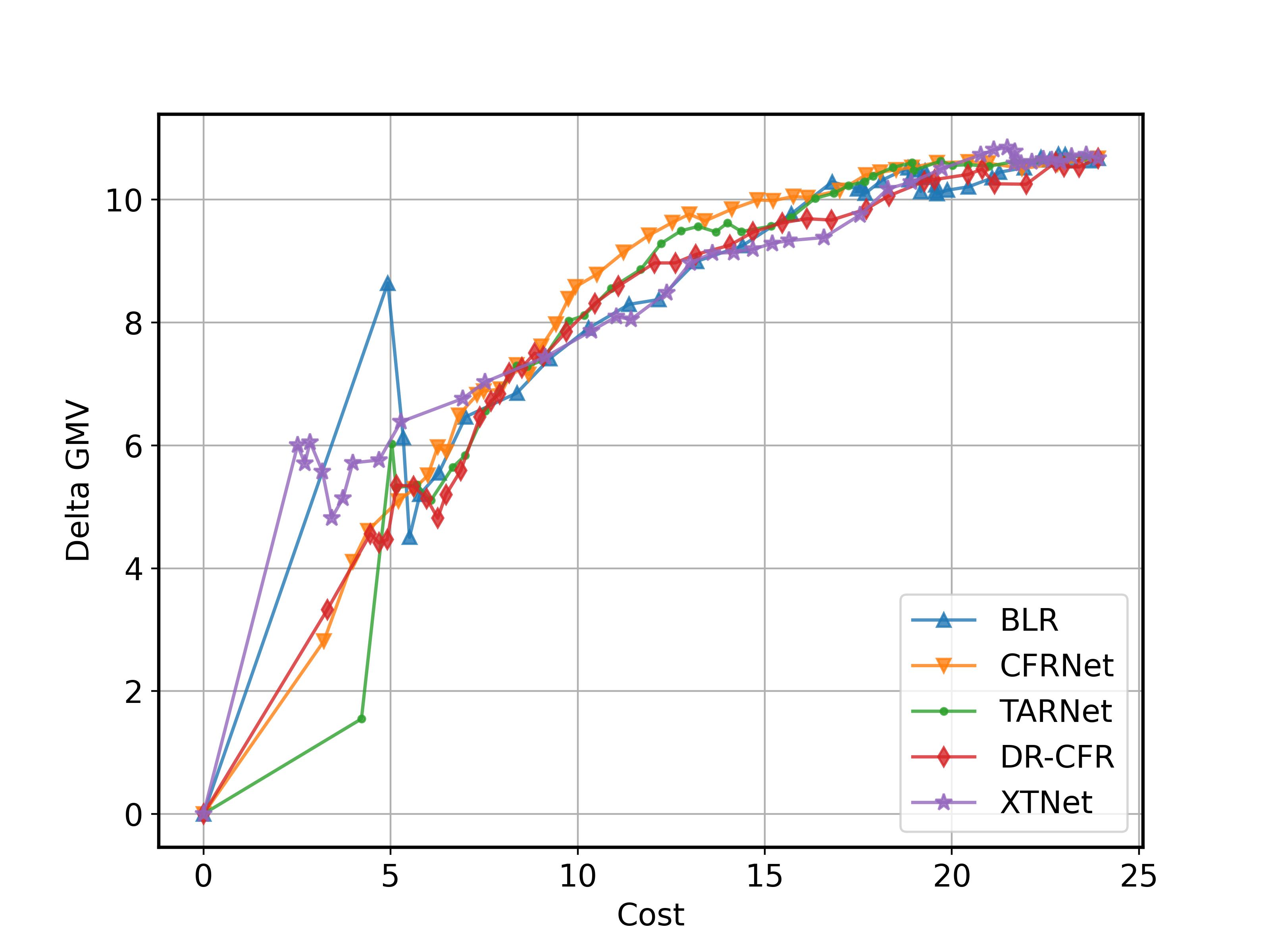}
    \end{minipage}
    &
    \begin{minipage}{0.33\textwidth}
      \centering
      \includegraphics[width=\linewidth, height=40mm]{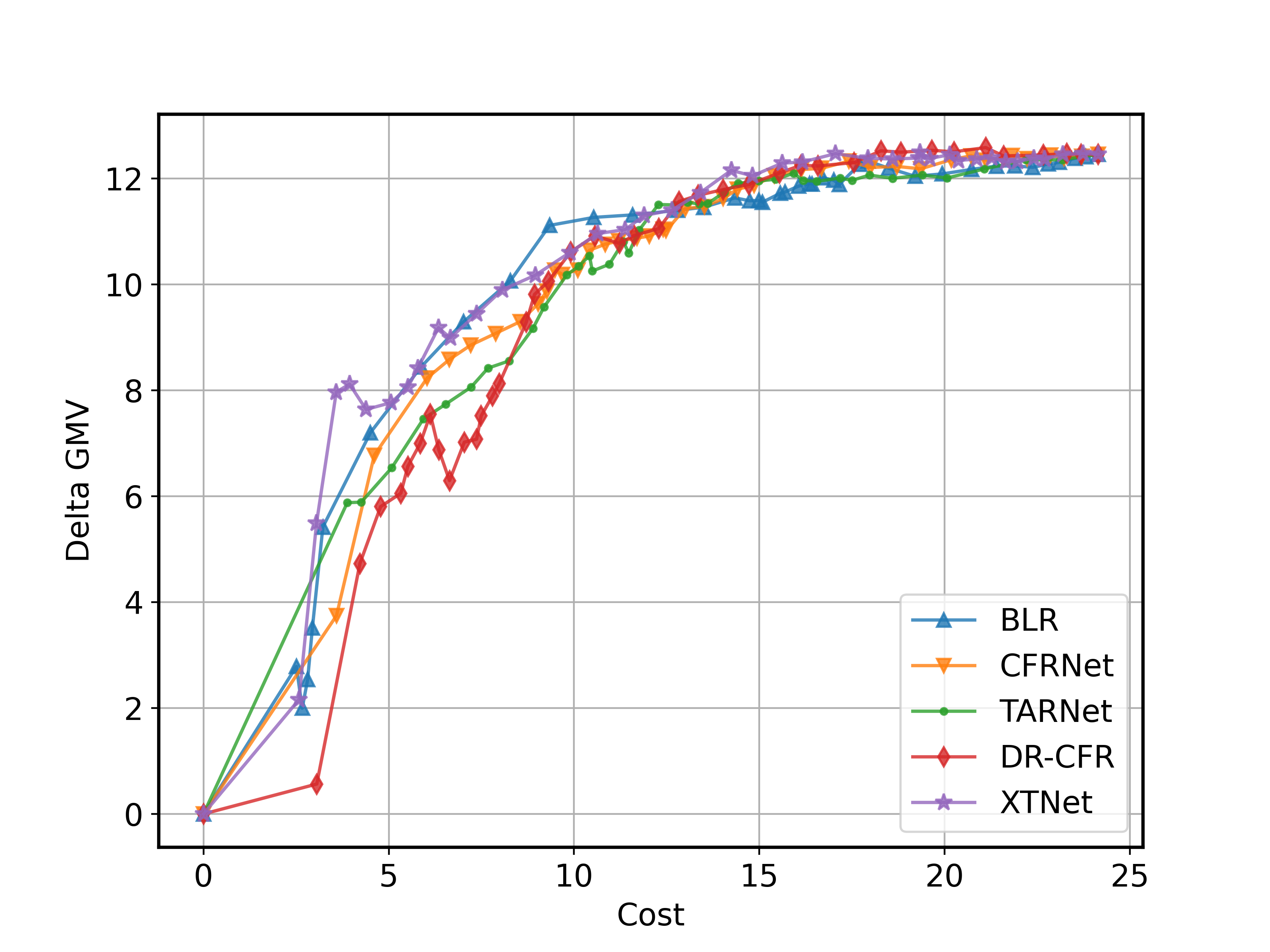}
    \end{minipage}
    \\
    \bottomrule
    \end{tabularx}
    \caption{MCMV-AUCC curves across synthetic datasets demonstrating XTNet's superior performance.}
    \vspace{-20pt}
    \label{tab:mcmv-curves}
\end{table*}

\begin{table}[t]
    \begin{tabularx}{0.47\textwidth}{@{}lXcc@{}}
    \toprule
    \textbf{Dataset}
    &\textbf{Configuration} 
    &\textbf{Ranking Error} 
    &\textbf{MCMV-AUCC}\\
    \midrule
    \multirow{4.0}{*}{Synthetic-1} & 
     w/o $\mathcal{L}_{imb}$ [BEM]
    & $0.2609 \pm 0.0324$ & $ 0.4157$  \\
    \cmidrule(lr){2-4} &
     w/$\mathcal{L}_{imb} $ [EM]
    & $0.2687 \pm 0.0794$ & $0.0183$  \\
    \cmidrule(lr){2-4} &
     w/$\mathcal{L}_{imb} $ [BEM]
    & $\textbf{0.2272} \pm 0.0433$ & $\mathbf{0.5375}$  \\
    \midrule
    \multirow{4.0}{*}{Synthetic-2} & 
     w/o $\mathcal{L}_{imb}$ [BEM]
    & $0.2576 \pm 0.0264$ & $\mathbf{0.0692}$  \\
    \cmidrule(lr){2-4} &
     w/$\mathcal{L}_{imb} $ [EM]
    & $0.2687 \pm 0.0794$ & $0.0183$  \\
    \cmidrule(lr){2-4} &
     w/$\mathcal{L}_{imb}$ [BEM]
    & $\textbf{0.2465} \pm 0.0433$ & ${0.0529}$  \\
    \midrule
    \multirow{4.0}{*}{Synthetic-3} & 
     w/o $\mathcal{L}_{imb}$ [BEM]
    & $0.2751 \pm 0.0316$ & ${0.0474}$  \\
    \cmidrule(lr){2-4} &
     w/ $\mathcal{L}_{imb}$ [EM]
    & $0.2949 \pm 0.1785$ & $-0.0711$  \\
    \cmidrule(lr){2-4} &
     w/ $\mathcal{L}_{imb}$ [BEM]
    & $\textbf{0.2726} \pm 0.0335$ & $\mathbf{0.0834}$  \\
    \bottomrule
    \end{tabularx}
    \caption{Ablation study on the imbalance loss component and BasicNet component. [BEM] denotes BasicNet+EffectNet+MaskNet. [EM] denotes EffectNet+MaskNet.}
    \vspace{-30pt}
    \label{tab:ablation-study-imb}
\end{table}

\subsection{Main Results (RQ1)}

Tables~\ref{tab:perf-syn} and~\ref{tab:perf-real} present our main experimental results. XTNet consistently outperforms all baseline methods across both synthetic and real-world datasets. On synthetic datasets, XTNet achieves the lowest ranking error (0.2272 on Syn-1) and substantially higher MCMV-AUCC scores. The performance gains are particularly pronounced in MCMV-AUCC, demonstrating XTNet's superior ability to capture cost-effective treatment allocation in multi-category scenarios.

On the real-world dataset, XTNet maintains its competitive edge with the lowest ranking error (0.1100) and highest MCMV-AUCC (1.1701), validating the practical applicability of our approach. The substantial performance gap in MCMV-AUCC across all datasets indicates that XTNet more effectively captures the complex interactions between multi-category treatments, which is crucial for real-world deployment.

\subsection{Ablation Analysis (RQ2)}

Table~\ref{tab:ablation-study-imb} presents our ablation study examining the contribution of the imbalance loss term $\mathcal{L}_{imb}$. The results demonstrate that including the imbalance loss consistently reduces ranking error across all synthetic datasets, with improvements of 6\% on average. While the effect on MCMV-AUCC varies across datasets, the overall trend indicates that the imbalance loss enhances the model's ability to handle treatment selection bias, which is crucial for real-world applications with observational data.

We also conduct the ablation study of the BasicNet module. As shown in Table~\ref{tab:ablation-study-imb}, our BasicNet module is crucial for the overall performance and it cannot be dropped.

\subsection{Evaluation Metric Analysis (RQ3)}

Our theoretical analysis in Section 4 establishes that MCMV-AUCC provides lower metric error compared to traditional evaluation approaches for multi-category, multi-valued treatments. The empirical results validate this theoretical advantage: while baseline methods show inconsistent performance across different evaluation metrics, XTNet demonstrates robust superiority under MCMV-AUCC evaluation.

The substantial performance gaps observed in MCMV-AUCC compared to more modest gains in ranking error indicate that our proposed metric better captures the nuanced requirements of multi-category, multi-valued treatment optimization, particularly the cost-effectiveness considerations critical for practical deployment.

\begin{table}[t]
    \centering
    \begin{tabularx}{0.25\textwidth}{@{}lcc@{}}
    \toprule
    \textbf{Method} 
    & \textbf{GMV Gain} 
    & \textbf{Order Gain}\\
    \midrule
    Baseline & 0\% & 0\%  \\
    \midrule
    BLR & +2.43\% & +1.73\% \\
    \midrule
    XTNet & \textbf{+4.33}\% & \textbf{+2.10}\% \\
    \bottomrule
    \end{tabularx}
    \caption{Online A/B Test Results}
    \label{tab:online-ab-res}
    \vspace{-20pt}
\end{table}

\subsection{Online A/B Test}
We also conducted the A/B Test across 32 cities over 1 week on the ride-hailing platform. We use the trained network to estimate the finish rate with different discounts. As shown in Table~\ref{tab:online-ab-res}, our proposed XTNet achieves the highest GMV gain and Order gain among all models.

%% file: conclusion.tex
\section{Conclusion}


This paper addresses the challenging problem of counterfactual causal inference in multi-category, multi-valued treatment scenarios. We introduce XTNet, a novel deep neural architecture that captures complex cross-treatment interactions via dynamic masking mechanisms and decomposition strategies without restrictive assumptions. We propose MCMV-AUCC, a tailored evaluation metric that demonstrates lower metric error compared to traditional approaches for multi-dimensional treatment assessment. Our comprehensive experiments on synthetic and real-world datasets show that XTNet consistently outperforms state-of-the-art baselines. The real-world A/B test results further confirm its effectiveness.

%% file: appendix.tex
\section{Appendix}

\subsection{Qini v.s. AUCC Proofs}
\label{app:qini-aucc}
\begin{proof}
Consider the Qini allocation $S_Q(B) = \{(x_i,t_{j_i})\}$ for budget $B$, where treatments are ranked by $\mathbb{E}[Y(x,t)]$. Since Qini ignores cost, it typically selects higher-index treatments without efficiency consideration.

For each $(x_i, t_{j_i}) \in S_Q(B)$ with $j_i > 1$, define the budget-reallocated strategy: replace $(x_i, t_{j_i})$ with $(x_i, t_1)$ and use the saved budget $\Delta C_i = \mathbb{E}[C(x_i, t_{j_i})] - \mathbb{E}[C(x_i, t_1)] \geq 0$ to purchase the highest-RoI treatment available.

By Lemma \ref{lem:roi_ordering}, $\rho(x_i, t_1) > \rho(x_i, t_{j_i})$, which implies
\[
\mathbb{E}[Y(x_i, t_1)] \cdot \mathbb{E}[C(x_i, t_{j_i})] > \mathbb{E}[Y(x_i, t_{j_i})] \cdot \mathbb{E}[C(x_i, t_1)]
\]
Rearranging: $\frac{\mathbb{E}[Y(x_i, t_1)] \cdot \Delta C_i}{\mathbb{E}[C(x_i, t_1)]} > \mathbb{E}[Y(x_i, t_{j_i})] - \mathbb{E}[Y(x_i, t_1)]$.

Since the AUCC policy reallocates the saved budget $\Delta C_i$ to treatments with RoI at least $\rho(x_i, t_1)$, the net outcome improvement is strictly positive. Applying this argument to all treatments in $S_Q(B)$ yields $\mathbb{E}_{\Pi_{\text{AUCC}}}[O(B)] \geq \mathbb{E}_{\Pi_{\text{Qini}}}[O(B)]$ for all $B$. By Theorem \ref{thm:stoch_dominance}, $\mathcal{E}(M_{\text{AUCC}}) \leq \mathcal{E}(M_{\text{Qini}})$.    
\end{proof}

\subsection{MV-AUCC v.s. MCMV-AUCC}
\label{app:mv-aucc-mcmv-aucc}

We establish the superiority of MCMV-AUCC over MV-AUCC through an analysis of the fundamental limitations of local optimization in multi-dimensional treatment spaces.

\begin{assumption}[Interaction Effects Structure]
\label{ass:interaction_structure}
The outcome function has the additive form with bounded interaction effects:
\[
\mathbb{E}[Y(x,\mathbf{t})] = \sum_{j=1}^k f_j(x, t^{(j)}) + \sum_{1 \leq i < j \leq k} g_{ij}(x, t^{(i)}, t^{(j)})
\]
where:
\begin{enumerate}
\item $f_j(x, \cdot)$ are strictly increasing and concave for all $j \in \{1,\ldots,k\}$
\item $g_{ij}(x, s, t) \geq 0$ for all $s, t > 0$ (positive interactions)
\item $|g_{ij}(x, s, t)| \leq G < \infty$ for some constant $G$ (bounded interactions)
\item $g_{ij}(x, s, t) = 0$ if $s = 0$ or $t = 0$ (no interaction without both treatments)
\end{enumerate}
\end{assumption}

\begin{assumption}[Ordering Error Frequency]
\label{ass:ordering_error_freq}
For treatments $\mathbf{t}_A, \mathbf{t}_B$ with $Q(\mathbf{t}_A) = Q(\mathbf{t}_B) = q$ and $I(\mathbf{t}_A) = 0 < I(\mathbf{t}_B)$ where $I(\mathbf{t}) = \sum_{i<j} \mathbf{1}[t^{(i)} > 0, t^{(j)} > 0]$, the probability of ordering error satisfies:
\[
\mathbb{P}[\mu^*(x, \mathbf{t}_A) > \mu^*(x, \mathbf{t}_B) \text{ but } \mathbb{E}[Y(x, \mathbf{t}_B)] > \mathbb{E}[Y(x, \mathbf{t}_A)]] \geq \phi(k, G)
\]
where $\phi(k, G) = \min\left\{\frac{G}{2\max_j f_j'(0)}, \frac{1}{2}\right\}$ is the minimum ordering error probability.
\end{assumption}

\begin{assumption}[Treatment Space Density]
\label{ass:treatment_density}
For each intensity level $q \geq k$, the number of treatments with non-zero interaction effects satisfies:
\[
|\{\mathbf{t} : Q(\mathbf{t}) = q, I(\mathbf{t}) > 0\}| \geq \rho \cdot N(q)
\]
where $N(q) = |\{\mathbf{t} : Q(\mathbf{t}) = q\}|$ and $\rho \in (0,1]$ is the interaction density parameter.
\end{assumption}

\begin{theorem}[MCMV-AUCC Superiority with Concrete Bounds]
\label{thm:mcmv_superiority_concrete}
Under Assumptions \ref{ass:stoch_monotonicity}, \ref{ass:stoch_convexity}, \ref{ass:bounded_moments}, \ref{ass:grouping_coherence}, \ref{ass:interaction_structure}, \ref{ass:ordering_error_freq}, and \ref{ass:treatment_density}, MCMV-AUCC achieves lower expected metric error than MV-AUCC:
\[
\mathcal{E}(M_{\text{MCMV-AUCC}}) \leq \mathcal{E}(M_{\text{MV-AUCC}}) - \Delta_{\min}
\]
where $\Delta_{\min} = \rho \cdot \phi(k, G) \cdot \frac{G}{2} \cdot \mathbb{E}[B_{\max}]$ is the minimum improvement bound.
\end{theorem}

\begin{proof}
The superiority of MCMV-AUCC stems from MV-AUCC's systematic ordering errors caused by erroneous marginal RoI estimations within incomparable treatment groups.

Under Assumption \ref{ass:interaction_structure}, MV-AUCC computes marginal RoI as 
\begin{align}
    \mu^*(x,\mathbf{t}) &= \max_{j: t^{(j)} > 0} \frac{f_j(x, t^{(j)}) - f_j(x, t^{(j)} - 1)}{\mathbb{E}[C(x,\mathbf{t})] - \mathbb{E}[C(x,\mathbf{t} - \mathbf{e}_j)]} \nonumber \\
    &+ \frac{ \sum_{i \neq j} [g_{ij}(x, t^{(i)}, t^{(j)}) - g_{ij}(x, t^{(i)}, t^{(j)} - 1)]}{\mathbb{E}[C(x,\mathbf{t})] - \mathbb{E}[C(x,\mathbf{t} - \mathbf{e}_j)]}. \nonumber
\end{align} 
This calculation captures only partial interaction effects (those involving the decremented category $j$) while ignoring other interaction terms.

Consider treatments $\mathbf{t}_A = (q, 0, \ldots, 0)$ and $\mathbf{t}_B = (1, 1, \ldots, 1, q-k+1, 0, \ldots, 0)$ with the same intensity $q \geq k$. The total outcomes are:
\begin{align}
\mathbb{E}[Y(x, \mathbf{t}_A)] &= f_1(x, q) \\
\mathbb{E}[Y(x, \mathbf{t}_B)] &= \sum_{j=1}^{k-1} f_j(x, 1) + f_k(x, q-k+1) \\&+ \sum_{1 \leq i < j \leq k-1} g_{ij}(x, 1, 1) + \sum_{j=1}^{k-1} g_{jk}(x, 1, q-k+1)
\end{align}

By Assumption \ref{ass:interaction_structure}, the interaction terms in $\mathbf{t}_B$ contribute at least $\binom{k-1}{2} g_{\min} + (k-1)g_{\min} \geq (k-1)g_{\min}$ where $g_{\min} > 0$ is the minimum positive interaction value. Under concavity of $f_j$, we have $f_1(x, q) < q \cdot f_1'(0)$ and $\sum_{j=1}^{k-1} f_j(x, 1) + f_k(x, q-k+1) \geq (q-1) \cdot \min_j f_j'(0)$.

Therefore, when $g_{\min} > \frac{f_1'(0) - \min_j f_j'(0)}{\max\{k-1, 1\}}$, we have $\mathbb{E}[Y(x, \mathbf{t}_B)] > \mathbb{E}[Y(x, \mathbf{t}_A)]$. However, $\mu^*(x, \mathbf{t}_A) = \frac{f_1(x, q) - f_1(x, q-1)}{\Delta c_A}$ may exceed $\mu^*(x, \mathbf{t}_B)$ when $f_1$ exhibits strong marginal returns in the concentrated allocation.

By Assumption \ref{ass:ordering_error_freq}, this ordering error occurs with probability at least $\phi(k, G)$. The expected outcome difference when the error occurs is bounded below by $\frac{G}{2}$ (half the maximum interaction effect).

Under Assumption \ref{ass:treatment_density}, at least $\rho \cdot N(q)$ treatments at each intensity level $q$ have positive interactions. For any budget $B$, MV-AUCC's allocation includes approximately $\frac{B}{\bar{C}}$ treatments where $\bar{C}$ is the average treatment cost. The fraction of these at intensity levels with interaction opportunities is bounded below by $\rho$.

Therefore, the expected performance gap is:
\begin{align}
\mathbb{E}_{\Pi_{\text{MCMV}}}[O(B)] - \mathbb{E}_{\Pi_{\text{MV}}}[O(B)] &\geq \rho \cdot \phi(k, G) \cdot \frac{G}{2} \cdot \frac{B}{\bar{C}} \\
&= \rho \cdot \phi(k, G) \cdot \frac{G}{2\bar{C}} \cdot B
\end{align}

Integrating over all budgets $B \in [0, B_{\max}]$ yields:
\[
\mathcal{E}(M_{\text{MV-AUCC}}) - \mathcal{E}(M_{\text{MCMV-AUCC}}) \geq \rho \cdot \phi(k, G) \cdot \frac{G}{2\bar{C}} \cdot \frac{B_{\max}^2}{2} = \Delta_{\min}
\]

This establishes the concrete lower bound on the improvement achieved by MCMV-AUCC.
\end{proof}

{\renewcommand{\arraystretch}{1.2}%
\begin{table}[t]
\centering
\caption{Definitions of variables and formulas in the generation process of synthetic data 1}
\label{tab:data_generation_formulas}
\begin{tabular}{cl}
\toprule
Symbol & Formula \\
\midrule
$x^{\text{norm}}$ & $\displaystyle x^{\text{norm}} =  \frac{x - x_{\text{min}}}{x_{\text{max}} - x_{\text{min}}}$ \\
$\epsilon_{\text{gaussian}}$ & $\displaystyle \epsilon_{\text{gaussian}} \sim \mathcal{N}(0, \sigma)$ \\
$\epsilon_{\text{uniform}}$ &  $\displaystyle \epsilon_{\text{uniform}} \sim \mathcal{U}(a, b), \quad \sigma = r$ \\
$x_i'$ & $\displaystyle x_i' = x_i^{\text{norm}},\ i=1,\dots,8$ \\
$t$ & 
$
\begin{aligned}[t]
t &= \sum_{i=1}^5 a_ix_i' + \sum_{i=1}^4 a_ix_i'x_{i+1}' \\
  &\quad + x_1'^2 + \sin(x_3') + \exp(-x_5')
\end{aligned} 
$
\\
$t_{\text{noisy}}$ & $\displaystyle t_{\text{noisy}} = t + \epsilon_{\text{gaussian}} + \epsilon_{\text{uniform}}$ \\
$t$ (discrete) & 
$\displaystyle \text{bins} = \{b_0, b_1, \dots, b_{m}\},\ b_i = \frac{i}{m},\ i = 0, 1, \dots, m$ \\
& $\displaystyle t = \arg \max_{i} \{b_{i-1} < t_{\text{noisy}}^{\text{norm}} \leq b_i \}$ \\
$S$ & 
$\begin{aligned}[t]
S &= f(x_1', ..., x_8') = \sum_{i=1}^{8} a_i x_i' + a_9 x_1' x_2' + a_{10} x_3' x_4' \\
&\quad + a_{11} x_4' x_5' x_6' + a_{12} x_6' x_8' x_1' + a_{13} x_5'^2 \\
&\quad + a_{14} \sin(x_8') + a_{15} e^{-x_3'}
\end{aligned}$\\
$t_1^{\text{effect}}$ & 
$\begin{aligned}[t]
t_1^{\text{effect}} &= 3 \cdot w_{t1} \cdot \left(e^{-t_1/5} - e^{1/5}\right) \\
&\quad \cdot \left( \sin(x_3') - e^{-x_7'-x_5'} + \sqrt{x_3' + x_6'^2} \right)
\end{aligned} $\\
$t_2^{\text{effect}}$ & 
$\begin{aligned}[t]
t_2^{\text{effect}} &= -w_{t2} \cdot \left(\ln(5 t_2 + 0.5) - \ln(1.5)\right) \\
&\quad \cdot \left( \cos(x_3') - x_4' x_2' + x_6'^2 + |x_1'| \right)
\end{aligned}$ \\
$t_{1,2}^{\text{effect}}$ & 
$\begin{aligned}[t]
t_{1,2}^{\text{effect}} &= w_{t3} \cdot \ln((5 t_1 - 1)(t_1 + t_2 - 1)) \\
&\quad \cdot \left( -x_1' + x_2'^3 \right)
\end{aligned} $\\
$t^{\text{effect}}$ & 
$\displaystyle t^{\text{effect}} = t_1^{\text{effect}} + t_2^{\text{effect}} + t_{1,2}^{\text{effect}}$ \\
$S_{\text{noisy}}$ &  
$\displaystyle S_{\text{noisy}} = S + t^{\text{effect}} + \epsilon_{\text{gaussian}} + \epsilon_{\text{uniform}}$ \\
$y$ (continuous) & 
$\displaystyle y = S_{\text{noisy}}^{\text{norm}}$ \\
$y$ (discrete) & 
$\displaystyle \text{bins} = \{b_0, b_1, \dots, b_m\},\ b_i = \frac{i}{m},\ i = 0, 1, \dots, m$ \\
& $\displaystyle y = \arg \max_{i} \{ b_{i-1} < S_{\text{noisy}}^{\text{norm}} \leq b_i \}$ \\
\bottomrule
\end{tabular}
\end{table}
}

\subsection{Formulation of Synthetic Data}

The synthetic data generation process employs a comprehensive framework to simulate complex relationships and noise structures. Input features \(x\) are first normalized to \(x^{\text{norm}}\) using min-max scaling. The framework introduces both Gaussian (\(\epsilon_{\text{gaussian}} \sim \mathcal{N}(0, \sigma)\)) and uniform (\(\epsilon_{\text{uniform}} \sim \mathcal{U}(a, b)\)) noise components, where \(\sigma = r\) controls the noise magnitude. The intermediate target \(t\) combines linear terms (\(\sum a_ix_i'\)), pairwise interactions (\(\sum a_ix_i'x_{i+1}'\)), and nonlinear transformations (e.g., \(x_1'^2, \sin(x_3'), \exp(-x_5')\)), with \(t_{\text{noisy}}\) incorporating additive noise. For discrete outcomes, \(t_{\text{noisy}}^{\text{norm}}\) is binned into \(m\) categories. The composite score \(S\) extends this with higher-order interactions (e.g., \(x_4'x_5'x_6'\)) and additional nonlinearities (e.g., \(\sin(x_8'), e^{-x_3'}\)). Treatment effects (\(t_1^{\text{effect}}, t_2^{\text{effect}}, t_{1,2}^{\text{effect}}\)) are modeled through weighted combinations of logarithmic, trigonometric, and polynomial functions of treatments \(t_1, t_2\) and features, with \(t^{\text{effect}}\) representing their cumulative impact. The final output \(y\) is derived from \(S_{\text{noisy}}^{\text{norm}}\) (continuous) or its binned discretization, where \(S_{\text{noisy}}\) integrates the base score, treatment effects, and noise components. This design enables simulation of realistic data with configurable nonlinearities, noise levels, and treatment responses.

\subsubsection{Synthetic Data 1}

The synthetic data generation process employs a comprehensive framework to simulate complex relationships and noise structures. Input features \(x\) are first normalized to \(x^{\text{norm}}\) using min-max scaling. The framework introduces both Gaussian (\(\epsilon_{\text{gaussian}} \sim \mathcal{N}(0, \sigma)\)) and uniform (\(\epsilon_{\text{uniform}} \sim \mathcal{U}(a, b)\)) noise components, where \(\sigma = r\) controls the noise magnitude. The intermediate target \(t\) combines linear terms (\(\sum a_ix_i'\)), pairwise interactions (\(\sum a_ix_i'x_{i+1}'\)), and nonlinear transformations (\(x_1'^2, \sin(x_3'), \exp(-x_5')\)), with \(t_{\text{noisy}}\) incorporating additive noise. For discrete outcomes, \(t_{\text{noisy}}^{\text{norm}}\) is binned into \(m\) categories. The composite score \(S\) extends this with higher-order interactions (e.g., \(x_4'x_5'x_6'\)) and additional nonlinearities (\(\sin(x_8'), e^{-x_3'}\)). Treatment effects (\(t_1^{\text{effect}}, t_2^{\text{effect}}, t_{1,2}^{\text{effect}}\)) are modeled through weighted combinations of logarithmic, trigonometric, and polynomial functions of treatments \(t_1, t_2\) and features, with \(t^{\text{effect}}\) representing their cumulative impact. The final output \(y\) is derived from \(S_{\text{noisy}}^{\text{norm}}\) (continuous) or its binned discretization, where \(S_{\text{noisy}}\) integrates the base score, treatment effects, and noise components. This design enables simulation of realistic data with configurable nonlinearities, noise levels, and treatment responses.

\subsubsection{Synthetic Data 2}
The synthetic data 2 generation process represents a less complex alternative to the previous design, evidenced by: (1) fewer interaction terms (two versus four) and nonlinear transformations (two versus three) in \(S\); (2) simplified treatment effects without logarithmic or complex trigonometric components; and (3) reduced feature utilization in effect calculations, employing additive rather than multiplicative feature combinations. The streamlined structure facilitates computational efficiency while preserving sufficient complexity for model validation tasks.

\subsubsection{Synthetic Data 3}
This simplified framework differs from previous versions through: (1) complete removal of all nonlinear transformations (quadratic, trigonometric, exponential) in both target and score calculations; (2) elimination of all higher-order feature interactions; (3) reduction of treatment effect components from complex logarithmic/trigonometric functions to basic linear operations; and (4) simplified noise injection using only additive terms without combined noise effects. The resulting dataset maintains a deliberately elementary structure suitable for benchmarking basic model capabilities or serving as a control condition in methodological comparisons.

{\renewcommand{\arraystretch}{1.3}%
\begin{table}[t]
\centering
\caption{Definitions of variables and formulas in the generation process of synthetic data 2}
\label{tab:data_generation_formulas_1}
\begin{tabular}{cl}
\toprule
Symbol & Formula \\
\midrule
$x^{\text{norm}}$ & $\displaystyle x^{\text{norm}} =  \frac{x - x_{\text{min}}}{x_{\text{max}} - x_{\text{min}}}$ \\
$\epsilon_{\text{gaussian}}$ & $\displaystyle \epsilon_{\text{gaussian}} \sim \mathcal{N}(0, \sigma)$ \\
$\epsilon_{\text{uniform}}$ & $\displaystyle \epsilon_{\text{uniform}} \sim \text{Uniform}\left(a, b \right)$ \\
$x_i'$ & $\displaystyle x_i' = x_i^{\text{norm}},\ i=1,\dots,8$ \\
$t$ &
$\begin{aligned}[t]
t &= \sum_{i=1}^{3} a_i x_i' + \sum_{i=1}^{3} a_i x_i' x_{i+1}' + x_3'^2
\end{aligned} $\\
$t_{\text{noisy}}$ & $\displaystyle t_{\text{noisy}} = t + \epsilon_{\text{uniform}}$ \\
$t$ (discrete) &
$\displaystyle \text{bins} = \{b_0, b_1, \dots, b_m\},\ b_i = \frac{i}{m},\ i = 0, 1, \dots, m$ \\
& $\displaystyle t = \arg \max_{i} \{ b_{i-1} < t_{\text{noisy}}^{\text{norm}} \leq b_i \}$ \\
$t_i'$ & $\displaystyle t_i' = t_i^{\text{norm}},\ i = 1,2$ \\
$S$ &
$\begin{aligned}[t]
S &= f(x_1', ..., x_8') = \sum_{i=1}^{8} a_i x_i' + a_9 x_1' x_2' + a_{10} x_3' x_4' \\
&\quad + a_{13} x_5'^2 + a_{14} \sin(x_8')
\end{aligned}$ \\
$t_1^{\text{effect}}$ &
$\displaystyle t_1^{\text{effect}} = w_{t1} \cdot t_1' \cdot (x_1' + x_2' + x_3')$ \\
$t_2^{\text{effect}}$ &
$\displaystyle t_2^{\text{effect}} = -w_{t2} \cdot t_2' \cdot (x_4' + x_5' + x_6')$ \\
$t^{\text{effect}}$ &
$\begin{aligned}[t]
t^{\text{effect}} &= t_1^{\text{effect}} + t_2^{\text{effect}} \\
&\quad + 2 \cdot w_{t3} \cdot (t_1' \cdot t_2') \cdot (-x_7' + x_8')
\end{aligned} $\\
$S_{\text{noisy}}$ & $\displaystyle S_{\text{noisy}} = S + t^{\text{effect}} + \epsilon_{\text{gaussian}} + \epsilon_{\text{uniform}}$ \\
$y$ (continuous) & $\displaystyle y = S_{\text{noisy}}^{\text{norm}}$ \\
$y$ (discrete) &
$\displaystyle \text{bins} = \{b_0, b_1, \dots, b_m\},\ b_i = \frac{i}{m},\ i = 0, 1, \dots, m$ \\
& $\displaystyle y = \arg \max_{i} \{ b_{i-1} < S_{\text{noisy}}^{\text{norm}} \leq b_i \}$ \\
\bottomrule
\end{tabular}
\end{table}
}

{\renewcommand{\arraystretch}{1.3}%
\begin{table}[t]
\centering
\caption{Definitions of variables and formulas in the generation process of synthetic data 3}
\label{tab:data_generation_formulas_3}
\begin{tabular}{cl}
\toprule
Symbol & Formula \\
\midrule
$x^{\text{norm}}$ & $\displaystyle x^{\text{norm}} = \frac{x - x_{\text{min}}}{x_{\text{max}} - x_{\text{min}}}$ \\
$\epsilon_{\text{gaussian}}$ & $\displaystyle \epsilon_{\text{gaussian}} \sim \mathcal{N}(0, \sigma), \quad \sigma = r$ \\
$x_i'$ & $\displaystyle x_i' = x_i^{\text{norm}},\ i=1,\dots,8$ \\
$t$ &
$\displaystyle t = \sum_{i=1}^{3} a_i x_i'$ \\
$t_{\text{noisy}}$ & $\displaystyle t_{\text{noisy}} = t + \epsilon_{\text{uniform}}$ \\
$t$ (discrete) &
$\displaystyle \text{bins} = \{b_0, b_1, \dots, b_m\},\ b_i = \frac{i}{m},\ i = 0, 1, \dots, m$ \\
& $\displaystyle t = \arg \max_{i} \{ b_{i-1} < t_{\text{noisy}}^{\text{norm}} \leq b_i \}$ \\
$t_i'$ & 
$\displaystyle t_i' = t_i^{\text{norm}} ,\ i = 1,2 $ \\
$S$ & $\displaystyle S = f(x_1', ..., x_8') = \sum_{i=1}^{8} a_i x_i'$ \\
$t_1^{\text{effect}}$ & $\displaystyle t_1^{\text{effect}} = w_{t1} \cdot t_1' \cdot (x_1' + x_2')$ \\
$t_2^{\text{effect}}$ & $\displaystyle t_2^{\text{effect}} = -w_{t2} \cdot t_2' \cdot (x_4' + x_5')$ \\
$t^{\text{effect}}$ & 
$\begin{aligned}[t]
t^{\text{effect}} &= t_1^{\text{effect}} + t_2^{\text{effect}} \\
&\quad + 2 \cdot w_{t3} \cdot (t_1' \cdot t_2') \cdot (x_2' - x_6')
\end{aligned} $\\
$S_{\text{noisy}}$ & $\displaystyle S_{\text{noisy}} = S + t^{\text{effect}} + \epsilon_{\text{gaussian}}$ \\
$y$ (continuous) & $\displaystyle y = S_{\text{noisy}}^{\text{norm}}$ \\
$y$ (discrete) &
$\displaystyle \text{bins} = \{b_0, b_1, \dots, b_m\},\ b_i = \frac{i}{m},\ i = 0, 1, \dots, m$ \\
& $\displaystyle y = \arg \max_{i} \{ b_{i-1} < S_{\text{noisy}}^{\text{norm}} \leq b_i \}$ \\
\bottomrule
\end{tabular}
\end{table}
}

\subsection{Adaptations of Baseline Models}

Given the new scenario where treatments are categorized into \( m \) categories and each category has a multi-value treatment (with the \( k \)-th category having \( a_k \) values), we will adapt BLR, TarNet, CFRNet, and DR-CFR to accommodate this complexity.

In its original form, BLR concatenates a one-dimensional treatment variable in the middle layer to predict the outcome. This approach is straightforwardly extendable to multi-value treatments by adjusting the dimensionality of the concatenated treatment vector. To handle multiple categories of treatments, the model should concatenate an \( m \)-dimensional vector representing all treatment categories at the intermediate layer instead of just one. 

TarNet utilizes shared layers to learn common representations before employing a dual-head structure to estimate outcomes for treated and control groups separately. Transition from a dual-head to a multi-head architecture is suitable for MCMV-treatment scenario. And each head is designed to predict outcomes under specific combinations of treatments across all categories. The total number of heads would be \( \prod_{i}^{m} a_i \), reflecting the sum of all possible treatment value combinations across categories.

Building upon TarNet, CFRNet introduces additional loss terms to minimize the distance between the distributions of treated and control groups using integral probability metrics (IPM), such as the Wasserstein distance or Maximum Mean Discrepancy (MMD). Adaptation to CFRNet is similar to TarNet in structure, but with a focus on refining the IPM regularization. Specifically, the Wasserstein distance is extended to accommodate multi-dimensional treatments.

Moreover, an additional adaptation is introduced to BLR. Specifically, the IPM regularization is applied to the middle layer where the $ m $-dimensional treatment combination vector is concatenated. This modification aims to align the latent feature distributions between different treatment groups, thereby reducing selection bias inherent in observational data.

As for DR-CFR, we preserve its original causal graph structure while modifying the output architecture to enable counterfactual outcome prediction under complex treatment combinations. Specifically, the final output layer is restructured into a dual-tower framework, where one tower is responsible for estimating the baseline potential outcome under the control condition, and the other tower predicts individual-level sensitivity parameters that characterize the effect of each treatment component as well as their interactions. In the case of two-category multi-value treatments, the control tower predicts $\hat{y}^i(t^{(1)}=0, t^{(2)}=0)$, representing the counterfactual outcome when both treatments are at their baseline levels, while the parameter tower outputs the elasticity coefficients $z_1$ and $z_2$ corresponding to the individual’s responsiveness to treatment $t^{(1)}$ and $t^{(2)}$, along with $z_3$, which captures the interaction effect between the two treatments. The final predicted outcome under a specific treatment combination is then formulated as $y^i(t^{(1)}, t^{(2)}) = z_1 \cdot t^{(1)} + z_2 \cdot t^{(2)} + z_3 \cdot t^{(1)} \cdot t^{(2)}$, allowing for interpretable estimation of heterogeneous treatment effects in multi-dimensional treatment spaces.

\subsection{Implementation Details}

All experiments were conducted using PyTorch on NVIDIA 2080Ti GPUs. We used Adam optimizer with learning rate 0.01 and trained for 20 epochs. The loss coefficients were set to $\lambda_1 = 0.1$ and $\lambda_2 = 0.01$.